\documentclass[journal]{IEEEtran}

\usepackage{amsmath,amsfonts,amssymb,amsthm}
\usepackage{algorithmic}
\usepackage{algorithm}
\usepackage{graphicx}
\usepackage[colorlinks=true]{hyperref}
\usepackage{cite}
\usepackage{url}
\usepackage{bm}
\usepackage{mathtools}

\newtheorem{definition}{Definition}[section]
\newtheorem{assumption}{Assumption}[section]
\newtheorem{lemma}{Lemma}[section]
\newtheorem{theorem}{Theorem}[section]
\newtheorem{proposition}{Proposition}[section]
\newtheorem{corollary}{Corollary}[section]
\newtheorem{remark}{Remark}[section]

\newcommand{\E}{\mathbb{E}}
\newcommand{\Hcal}{\mathcal{H}}

\newcommand{\grad}{\nabla}
\newcommand{\cov}{\operatorname{Cov}}
\newcommand{\KL}{\operatorname{KL}}
\newcommand{\Var}{\operatorname{Var}}

\title{A Comparative Theoretical Analysis of Entropy Control Methods in Reinforcement Learning for Reasoning Language Models}

\author{
\IEEEauthorblockN{Ming Lei, Christophe Baehr} \\
\IEEEauthorblockA{
\textit{Shanghai Jiao Tong University, Shanghai} \\
\textit{Météo France/CNRS, Université de Toulouse, CNRM UMR 3589, Toulouse}\\
mlei@sjtu.edu.cn, christophe.baehr@math.univ-toulouse.fr}
}

\begin{document}
\maketitle

\begin{abstract}
Reinforcement learning (RL) has become a key approach for enhancing reasoning in large language models (LLMs), yet scalable training is often hindered by the rapid collapse of policy entropy, which leads to premature convergence and performance saturation \cite{cui2025entropymechanismreinforcementlearning}. This paper provides a comparative theoretical analysis of two entropy control strategies: traditional entropy regularization and the recently proposed covariance‑based mechanism. We establish a unified framework for entropy dynamics under softmax parameterization, showing that entropy change is governed by the covariance between log‑probabilities and logit updates. Our analysis reveals that traditional entropy regularization introduces a dense, persistent bias that modifies the stationary condition, leading to suboptimal policies, while covariance‑based methods selectively regularize a sparse subset of high‑covariance tokens and achieve asymptotic unbiasedness when the regularization coefficient is annealed. Moreover, covariance‑based methods preserve the stability margin of the base policy gradient, unlike traditional regularization. These results provide principled guidelines for entropy control in LLM post‑training, with implications for scaling RL to larger models and more complex reasoning tasks.
\end{abstract}

\section{Introduction}
\label{sec:intro}
Reinforcement learning (RL) has emerged as a dominant paradigm for post‑training large language models (LLMs), enabling them to acquire complex reasoning abilities that surpass purely supervised fine‑tuning. Recent breakthroughs, exemplified by OpenAI o1 \cite{openai2024o1} and DeepSeek‑R1 \cite{deepseek2025r1}, demonstrate that RL with verifiable rewards elicits sophisticated chain‑of‑thought behaviors and improves performance on challenging mathematical and coding tasks. However, scaling RL to these models reveals a fundamental challenge: the rapid collapse of policy entropy during training, which leads to premature convergence and performance saturation \cite{cui2025entropymechanismreinforcementlearning}.

Policy entropy, which quantifies the uncertainty in action selection, plays a critical role in balancing exploitation and exploration \cite{sutton1988learning}. In traditional RL, entropy regularization \cite{ziebart2008maximum, haarnoja2018soft} encourages exploration by adding an entropy bonus to the objective. While effective in many domains, this global regularization proves inadequate for reasoning LLMs: it either fails to prevent entropy collapse or introduces excessive bias that degrades final performance \cite{cui2025entropymechanismreinforcementlearning}. Recent empirical work \cite{cui2025entropymechanismreinforcementlearning} reveals that entropy collapse is driven by a small fraction of tokens exhibiting extremely high covariance between log‑probabilities and advantages. Based on this insight, the authors propose a covariance‑based entropy mechanism that selectively regularizes these high‑covariance tokens, implemented as Clip‑Cov (gradient detachment) and KL‑Cov (KL penalty). Their experiments show that these methods significantly mitigate entropy collapse while maintaining training stability and achieving superior downstream performance.

Despite these empirical successes, a rigorous theoretical understanding of why and how covariance‑based methods outperform traditional entropy regularization remains lacking. This paper fills this gap by providing a comprehensive theoretical analysis that:
\begin{enumerate}
\item Establishes a unified mathematical framework for entropy dynamics under softmax policy parameterization, deriving exact expressions for entropy change in terms of covariance between log‑probabilities and logit updates.
\item Compares the structural, convergence, and stability properties of traditional entropy regularization and covariance‑based methods, proving that the latter achieve asymptotic unbiasedness and preserve stability margins.
\item Provides theoretical guidelines for entropy control strategies based on problem characteristics, with implications for scalable LLM post‑training.
\end{enumerate}

Our analysis confirms that the covariance‑based approach is theoretically well‑grounded and offers distinct advantages over traditional regularization for reasoning tasks, where optimal policies are near‑deterministic and stability is paramount.

\subsection*{Paper Organization}
The remainder of this paper is structured as follows. Section~\ref{sec:related} reviews related work on entropy regularization, RL for LLMs, and the covariance‑based entropy mechanism. Section~\ref{sec:preliminaries} introduces the necessary notation and background. Section~\ref{sec:foundations} derives the foundational entropy dynamics for softmax policies under policy gradient. Section~\ref{sec:trad} analyzes traditional entropy regularization, highlighting its limitations. Section~\ref{sec:cov} presents the covariance‑based entropy mechanism and its theoretical properties. Section~\ref{sec:comparison} provides a comparative analysis of the two approaches. Section~\ref{sec:empirical} validates the theoretical predictions using empirical results from \cite{cui2025entropymechanismreinforcementlearning}. Section~\ref{sec:conclusion} concludes the paper. All detailed proofs are provided in the Appendix.

\section{Related Work}
\label{sec:related}

\subsection{Entropy Regularization in Reinforcement Learning}
Entropy regularization has a long history in RL, originating from the maximum entropy principle \cite{ziebart2008maximum} and later incorporated into policy gradient methods \cite{williams1991function} and actor‑critic algorithms \cite{schulman2017proximal}. Soft Actor‑Critic (SAC) \cite{haarnoja2018soft} demonstrated that maximizing a trade‑off between reward and entropy leads to improved exploration and robustness. These methods add an entropy term $+\alpha \Hcal(\pi_\theta)$ to the objective, encouraging stochasticity. In the context of RL for LLMs, entropy regularization has been applied in works such as InstructGPT \cite{ouyang2022training} and Llama 2 \cite{touvron2023llama}, often to prevent mode collapse. However, recent large‑scale reasoning models like DeepSeek‑R1 \cite{deepseek2025r1} and DAPO \cite{yu2025dapo} have noted that standard entropy regularization can be ineffective or even harmful, motivating the search for alternative exploration mechanisms.

\subsection{Reinforcement Learning for Large Language Models}
RL has become a standard component of LLM post‑training, particularly for aligning models with human preferences via RLHF \cite{ouyang2022training, christiano2017deep} and for enhancing reasoning with verifiable rewards \cite{openai2024o1, deepseek2025r1, cui2025processreinforcementimplicitrewards}. These approaches typically employ policy gradient algorithms such as PPO \cite{schulman2017proximal} or GRPO \cite{shao2024deepseekmathpushinglimitsmathematical}. Despite their empirical success, the underlying optimization dynamics—especially the role of policy entropy—remain underexplored. Recent work \cite{yu2025dapo, liu2025understanding} has begun to analyze training stability and entropy collapse, but a unified theoretical framework has been lacking.

\subsection{Predictability and Scaling in RL for LLMs}
The idea that RL training dynamics can be predicted has been explored in the context of scaling laws \cite{kaplan2020scaling, hoffmann2022training} and reward model overoptimization \cite{gao2022scaling}. \cite{cui2025entropymechanismreinforcementlearning} extended this line of work by showing that the relationship between validation performance and policy entropy follows an exponential law $R = -a\exp(\Hcal)+b$, which holds across models, tasks, and algorithms. This empirical law suggests a fundamental trade‑off that any entropy control method must address.

\subsection{Covariance‑Based Entropy Control}
The covariance‑based entropy mechanism \cite{cui2025entropymechanismreinforcementlearning} is a recent innovation that builds on the observation that entropy collapse is driven by a small set of tokens with high covariance between log‑probability and advantage. By selectively regularizing these tokens via gradient detachment (Clip‑Cov) or KL penalty (KL‑Cov), the method maintains exploration without sacrificing performance. This approach has been shown to outperform traditional entropy regularization on mathematical reasoning benchmarks, yet its theoretical foundations have not been formally analyzed—a gap this paper addresses.

\section{Preliminaries and Notation}
\label{sec:preliminaries}

\subsection{Reinforcement Learning Formulation}
We consider a standard RL setup where a policy $\pi_\theta$ generates responses for given prompts. Let $\mathcal{X}$ denote the input space and $\mathcal{Y}$ the output space. For a prompt $\boldsymbol{x} \in \mathcal{X}$, the policy autoregressively generates a response $\boldsymbol{y} = (y_1, \ldots, y_T) \sim \pi_\theta(\cdot|\boldsymbol{x})$. The objective is to maximize expected reward:
\begin{align}
J(\theta) = \E_{\boldsymbol{x} \sim \mathcal{D}, \boldsymbol{y} \sim \pi_\theta(\cdot|\boldsymbol{x})} \left[r(\boldsymbol{y})\right]
\label{eq:objective}
\end{align}
where $\mathcal{D}$ is the training distribution and $r(\boldsymbol{y}) \in \mathbb{R}$ is a verifiable reward (e.g., correctness for math problems).

\subsection{Softmax Policy Parameterization}
For language models, the policy is parameterized as a softmax over logits. For a given state $s$ (prompt and prefix) and action $a$ (next token), the policy is:
\begin{align}
\pi_\theta(a|s) = \frac{\exp(z_{s,a})}{\sum_{a' \in \mathcal{A}} \exp(z_{s,a'})}
\label{eq:softmax}
\end{align}
where $z_{s,a} = \theta_{s,a}$ is the logit parameter for state-action pair $(s,a)$. For theoretical analysis, we adopt the tabular softmax formulation, where each state-action pair has an independent parameter, following the standard policy gradient literature \cite{agarwal2021theory}.

\subsection{Policy Entropy}
The entropy of a policy $\pi$ at state $s$ is defined as:
\begin{align}
\nonumber
\Hcal(\pi(\cdot|s)) &= -\sum_{a \in \mathcal{A}} \pi(a|s) \log \pi(a|s) \\
&= \E_{a \sim \pi(\cdot|s)} \left[-\log \pi(a|s)\right]
\label{eq:entropy_state}
\end{align}
The average entropy over the state visitation distribution is:
\begin{align}
\Hcal(\pi_\theta) = \E_{s \sim d_{\pi_\theta}} \left[\Hcal(\pi_\theta(\cdot|s))\right]
\label{eq:entropy_avg}
\end{align}
where $d_{\pi_\theta}$ is the stationary state distribution induced by policy $\pi_\theta$.

\subsection{Policy Gradient and Advantage}
The policy gradient for the objective $J(\theta)$ is given by the policy gradient theorem \cite{williams1992simple}:
\begin{align}
\grad_\theta J(\theta) = \E_{\tau \sim \pi_\theta} \left[\sum_{t=0}^{T} \grad_\theta \log \pi_\theta(a_t|s_t) A^{\pi_\theta}(s_t, a_t)\right]
\label{eq:pg}
\end{align}
where $A^{\pi_\theta}(s_t, a_t) = Q^{\pi_\theta}(s_t, a_t) - V^{\pi_\theta}(s_t)$ is the advantage function. Here $Q^{\pi_\theta}(s_t, a_t)$ denotes the expected cumulative reward obtained by taking action $a_t$ in state $s_t$ and thereafter following policy $\pi_\theta$ (the action-value function), $V^{\pi_\theta}(s_t)$ denotes the expected cumulative reward obtained by following policy $\pi_\theta$ from state $s_t$ (the state-value function). Their difference $A^{\pi_\theta}(s_t, a_t)$ measures how much better the action is compared to the average. The advantage satisfies:
\begin{align}
\E_{a \sim \pi_\theta(\cdot|s)} [A^{\pi_\theta}(s, a)] = 0, \quad \forall s
\label{eq:adv_zero}
\end{align}

\section{Foundations of Entropy Dynamics}
\label{sec:foundations}

\subsection{Entropy Gradient for Softmax Policies}
We first derive the gradient of policy entropy with respect to policy parameters, which is fundamental to understanding entropy dynamics.

\begin{lemma}[Entropy Gradient]
\label{lem:entropy_grad}
For a softmax policy parameterized as in \eqref{eq:softmax}, the gradient of the state‑wise entropy $\Hcal_s(\theta)=\Hcal(\pi_\theta(\cdot| s))$ with respect to the logit $z_{s,a}$ is
\begin{align}
\frac{\partial \Hcal_s(\theta)}{\partial z_{s,a}} = -\pi_\theta(a| s)\Bigl(\log\pi_\theta(a| s)-\mathbb{E}_{a'\sim\pi_\theta(\cdot| s)}\bigl[\log\pi_\theta(a'| s)\bigr]\Bigr).
\end{align}
\end{lemma}

\begin{proof}
See Appendix~\ref{app:proof_entropy_grad}.
\end{proof}

\subsection{Entropy Change Under Parameter Update}
We now characterize how policy entropy changes after a single parameter update.

\begin{lemma}[First-Order Entropy Change]
\label{lem:entropy_change}
Let $\pi_\theta$ be a softmax policy as in \eqref{eq:softmax}. For a parameter update $\Delta z_{s,a}$, the first‑order change in the state‑wise entropy $\Hcal_s(\theta)=\Hcal(\pi_\theta(\cdot| s))$ satisfies
\begin{align}
\nonumber
&\Hcal_s(\theta+\Delta\theta)-\Hcal_s(\theta) \\
&= -\cov_{a\sim\pi_\theta(\cdot| s)}\bigl(\log\pi_\theta(a| s),\,\Delta z_{s,a}\bigr)+o(\|\Delta z\|),
\end{align}
where $\cov(X,Y)=\E[(X-\E[X])(Y-\E[Y])]$ and $o(\|\Delta z\|)$ collects higher‑order terms.
\end{lemma}

\begin{proof}
See Appendix~\ref{app:proof_entropy_change}.
\end{proof}

\subsection{Policy Gradient Update}
For policy gradient updates, the logit change is proportional to the advantage weighted by the policy probability.

\begin{proposition}[Policy Gradient Logit Update]
\label{prop:pg_update}
For a softmax policy $\pi_\theta$ as in \eqref{eq:softmax}, under the policy gradient update with learning rate $\eta$, the change in logit $z_{s,a}$ is
\begin{equation}
\Delta z_{s,a} = \eta \cdot \pi_\theta(a | s) \cdot A^{\pi_\theta}(s, a).
\end{equation}
\end{proposition}

\begin{proof}
See Appendix~\ref{app:proof_pg_update}.
\end{proof}

Combining Lemma~\ref{lem:entropy_change} and Proposition~\ref{prop:pg_update}, we obtain the fundamental entropy dynamics.

\begin{theorem}[Entropy Dynamics Under Policy Gradient]
\label{thm:entropy_dynamics}
For a softmax policy updated via policy gradient, the first‑order change in state‑wise entropy is
\begin{align}
\Delta \Hcal_s \approx -\eta\;\operatorname{Cov}_{a\sim\pi_\theta(\cdot| s)}\Bigl(\log\pi_\theta(a| s),\; \pi_\theta(a|s)\,A^{\pi_\theta}(s,a)\Bigr).
\end{align}
\end{theorem}

\begin{proof}
See Appendix~\ref{app:entro_dyn_under_grad}.
\end{proof}

\begin{remark}
Theorem~\ref{thm:entropy_dynamics} reveals that entropy decrease is driven by a positive covariance between log-probability and the product of probability and advantage. When high-probability actions also have high advantage (well-calibrated policy), the covariance is positive, leading to monotonic entropy reduction. Conversely, a negative covariance would increase entropy.
\end{remark}

\section{Traditional Entropy Regularization}
\label{sec:trad}

\subsection{Formulation}
Traditional entropy regularization incorporates an entropy bonus into the objective function \cite{williams1991function, schulman2017proximal, haarnoja2018soft}. The regularized objective is:
\begin{align}
J_{\text{reg}}(\theta) = J(\theta) + \alpha \Hcal(\pi_\theta)
\label{eq:reg_objective}
\end{align}
where $\alpha > 0$ is the entropy coefficient.

\subsection{Effect on Policy Gradient}
The entropy regularization modifies the effective policy gradient:
\begin{proposition}[Regularized Policy Gradient]
\label{prop:reg_grad}
The gradient of the entropy‑regularized objective is
\begin{align}
\nabla_\theta J_{\mathrm{reg}}(\theta) = \nabla_\theta J(\theta) + \alpha\,\nabla_\theta \mathcal{H}(\pi_\theta).
\end{align}
For a softmax policy, the gradient of entropy with respect to the logit $z_{s,a}$ is given by Lemma~\ref{lem:entropy_grad}.
\end{proposition}

\begin{proof}
See Appendix~\ref{app:regular_policy_grad}.
\end{proof}

The entropy gradient term introduces a bias that encourages higher entropy. Substituting into the logit update:
\begin{align}
\nonumber
\Delta z_{s,a}^{\text{reg}} &= \eta \cdot \Bigl( \pi_\theta(a|s) A^{\pi_\theta}(s, a) \\
&\quad - \alpha \pi_\theta(a|s) \bigl( \log \pi_\theta(a|s) - \E[\log \pi_\theta] \bigr) \Bigr)
\label{eq:reg_update}
\end{align}

\subsection{Entropy Dynamics Under Regularization}
The entropy change under regularization can be derived by substituting the regularized update into Lemma~\ref{lem:entropy_change}.

\begin{theorem}[Regularized Entropy Dynamics]
\label{thm:reg_entropy}
Under entropy regularization with coefficient $\alpha$, the first‑order change in state‑wise policy entropy satisfies
\begin{align}
\Delta \Hcal_s^{\text{reg}} &\approx -\eta \cdot \cov_{a \sim \pi_\theta(\cdot|s)}\Bigl(\log \pi_\theta(a|s),\; \pi_\theta(a|s) A^{\pi_\theta}(s, a)\Bigr) \notag \\
&\quad + \alpha \eta \cdot \Var_{a \sim \pi_\theta(\cdot|s)}\Bigl(\log \pi_\theta(a|s)\Bigr),
\end{align}
where the approximation is accurate when the policy is nearly deterministic, a regime typical of entropy collapse \cite{cui2025entropymechanismreinforcementlearning}.
\end{theorem}

\begin{proof}
See Appendix~\ref{app:proof_reg_entropy}.
\end{proof}

The variance term $\Var(\log \pi_\theta(a|s))$ is positive and increases as the policy becomes more deterministic. This term counteracts entropy collapse, explaining why entropy regularization can maintain higher entropy. However, it introduces a global bias that may interfere with optimal policy learning.

\subsection{Limitations of Traditional Entropy Regularization}
Despite its theoretical appeal, traditional entropy regularization faces several limitations in the context of reasoning LMs:

\begin{theorem}[Suboptimality of Global Entropy Regularization]
\label{thm:global_suboptimal}
Let $\pi^* = \arg\max_{\pi} \E_{\pi}[r]$ be an optimal policy maximizing the unregularized expected reward. For any $\alpha > 0$, define $\pi_{\text{reg}}^* = \arg\max_{\pi} \bigl( \E_{\pi}[r] + \alpha \Hcal(\pi) \bigr)$ as the optimal policy under entropy regularization. Then
\begin{align}
\E_{\pi_{\text{reg}}^*}[r] \le \E_{\pi^*}[r]
\end{align}
If $\pi^*$ is not a maximum‑entropy policy among all optimal policies, the inequality is strict. Moreover, the suboptimality gap satisfies
\begin{align}
\E_{\pi^*}[r] - \E_{\pi_{\text{reg}}^*}[r] \ge \alpha \bigl( \Hcal(\pi_{\text{reg}}^*) - \Hcal(\pi^*) \bigr).
\end{align}
\end{theorem}

\begin{proof}
See Appendix~\ref{app:proof_global_suboptimal}.
\end{proof}

\begin{corollary}[Sensitivity to Hyperparameters]
\label{cor:sensitivity}
The performance of entropy-regularized training is highly sensitive to the choice of $\alpha$:
\begin{enumerate}
    \item If $\alpha$ is too small, the regularization fails to counteract the natural entropy collapse driven by positive covariance (Theorem~\ref{thm:entropy_dynamics}), leading to premature saturation.
    \item If $\alpha$ is too large, the entropy gradient bias dominates the update, forcing a trade-off between entropy and reward that yields a significant suboptimality gap (Theorem~\ref{thm:global_suboptimal}) and reduces training stability (Theorem~\ref{them:stability}).
\end{enumerate}
Consequently, only a narrow range of $\alpha$ yields near-optimal performance.
\end{corollary}

\begin{proof}
See Appendix~\ref{app:sensi_hyperpara}.
\end{proof}

Empirical evidence \cite{cui2025entropymechanismreinforcementlearning} confirms this sensitivity, showing that only a narrow range of $\alpha$ values yields reasonable performance.

\section{Covariance-Based Entropy Mechanism}
\label{sec:cov}

\subsection{Theoretical Foundation}
The covariance-based entropy mechanism, proposed by \cite{cui2025entropymechanismreinforcementlearning}, builds on the insight that entropy collapse is driven by a small fraction of high-covariance tokens. The key idea is to selectively regularize these tokens rather than imposing global entropy constraints.

\begin{definition}[Token-Wise Covariance]
For a given state-action pair $(s, a)$ (corresponding to a token at a specific position), define the token-wise covariance contribution as:
\begin{align}
C(s, a) = \left(\log \pi_\theta(a|s) - \mu_{\log}(s)\right) \left(\Delta z_{s,a} - \mu_{\Delta z}(s)\right)
\label{eq:token_cov}
\end{align}
where $\mu_{\log}(s) = \E_{a' \sim \pi_\theta(\cdot|s)}[\log \pi_\theta(a'|s)]$ and $\mu_{\Delta z}(s) = \E_{a' \sim \pi_\theta(\cdot|s)}[\Delta z_{s,a'}]$.
\end{definition}
Note that $\E_{a \sim \pi_\theta}[C(s, a)] = \cov(\log \pi_\theta, \Delta z_{s,a})$, which directly determines entropy change via Lemma~\ref{lem:entropy_change}.

\subsection{Clip-Cov: Gradient Detachment}
The Clip-Cov mechanism selectively detaches gradients for high-covariance tokens.

\begin{definition}[Clip-Cov Operator]
Let $I_{\text{clip}}$ be a set of indices selected uniformly at random from tokens satisfying $C(s, a) \in [\omega_{\text{low}}, \omega_{\text{high}}]$, where $|I_{\text{clip}}| = r \cdot N$ with $r \ll 1$. The Clip-Cov policy loss is:
\begin{align}
L_{\text{Clip-Cov}}(\theta) = \begin{cases}
\E_t\left[\frac{\pi_\theta(y_t|\boldsymbol{y}_{<t})}{\pi_{\text{old}}(y_t|\boldsymbol{y}_{<t})} A_t\right], & t \notin I_{\text{clip}} \\
0, & t \in I_{\text{clip}}
\end{cases}
\label{eq:clip_cov}
\end{align}
\end{definition}

The gradient detachment effectively removes the contribution of high-covariance tokens from the update.

\begin{proposition}[Effect of Clip-Cov on Covariance]
\label{prop:clip_effect}
Under Clip-Cov with selection ratio $r$, the effective covariance used for policy updates becomes:
\begin{align}
\cov_{\text{eff}} = \cov_{\text{orig}} - \frac{r}{1-r} \Bigl( \E[C(s, a)|a \in I_{\text{clip}}] - \cov_{\text{orig}} \Bigr)
\label{eq:clip_effect}
\end{align}
When $\E[C(s, a)|a \in I_{\text{clip}}] > \cov_{\text{orig}}$, the effective covariance is strictly less than the original covariance.
\end{proposition}
\begin{proof}
See Appendix~\ref{app:proof_clip_effect}.
\end{proof}

\subsection{KL-Cov: KL Penalty on High-Covariance Tokens}
The KL-Cov mechanism applies a KL divergence penalty specifically to high-covariance tokens.

\begin{definition}[KL-Cov Operator]
Let $I_{\text{KL}}$ be the set of indices corresponding to the top $k$ proportion of tokens ranked by $|C(s, a)|$, with $k \ll 1$. The KL-Cov policy loss is:
\begin{align}
L_{\text{KL-Cov}}(\theta) = \begin{cases}
\E_t\left[\frac{\pi_\theta(y_t|\boldsymbol{y}_{<t})}{\pi_{\text{old}}(y_t|\boldsymbol{y}_{<t})} A_t\right], & t \notin I_{\text{KL}} \\
\E_t\left[\frac{\pi_\theta(y_t|\boldsymbol{y}_{<t})}{\pi_{\text{old}}(y_t|\boldsymbol{y}_{<t})} A_t + M \right], & t \in I_{\text{KL}}
\end{cases}
\label{eq:kl_cov}
\end{align}
\end{definition}
where $M = - \beta \KL(\pi_{\text{old}}(y_t|\boldsymbol{y}_{<t}) \| \pi_\theta(y_t|\boldsymbol{y}_{<t}))$.

\begin{theorem}[Regularization Effect of KL-Cov]
\label{thm:kl_effect}
For a token $a$ with high covariance $C(s, a) > 0$, the KL penalty induces an effective modification to the logit update:
\begin{align}
\Delta z_{s,a}^{\text{KL-Cov}} = \eta \pi_\theta(a|s) A^{\pi_\theta}(s, a) - \eta \beta \frac{\partial \KL(\pi_{\text{old}} \| \pi_\theta)}{\partial z_{s,a}}
\label{eq:kl_update}
\end{align}
where the KL gradient contributes a term proportional to $(\pi_\theta(a|s) - \pi_{\text{old}}(a|s))$. This term opposes updates that would increase covariance, effectively regularizing the entropy dynamics.
\end{theorem}
\begin{proof}
See Appendix~\ref{app:proof_kl_effect}.
\end{proof}

\subsection{Entropy Dynamics Under Covariance-Based Control}
We now derive the entropy dynamics under covariance-based control.

\begin{theorem}[Entropy Dynamics Under KL‑Cov]
\label{thm:klcov_entropy}
Let $\pi_\theta$ be a softmax policy as in \eqref{eq:softmax}, and let $I_{\mathrm{KL}}$ denote the set of tokens selected for the KL penalty, i.e., the top $k$ proportion (with $k\ll1$) of tokens ranked by $|C(s,a)|$, where $C(s,a)$ is defined in \eqref{eq:token_cov}. The KL‑Cov update is
\begin{align*}
\Delta z_{s,a}^{\mathrm{KL-Cov}} = 
\begin{cases}
\Omega - \eta\beta\bigl(\pi_\theta(a| s)-\pi_{\mathrm{old}}(a| s)\bigr), & (s,a)\in I_{\mathrm{KL}}\\
\Omega, & \text{otherwise}.
\end{cases}
\end{align*}
where $\Omega = \eta\,\pi_\theta(a| s)\,A^{\pi_\theta}(s,a)$.
Then the first‑order change in the state‑wise entropy $\mathcal{H}_s = \mathcal{H}(\pi_\theta(\cdot| s))$ satisfies
\begin{align}
\nonumber
\Delta \mathcal{H}_s^{\mathrm{KL-Cov}} &\approx -\eta\operatorname{Cov}_{a\sim\pi_\theta(\cdot| s)}
 \Bigl(\log\pi_\theta(a| s),\pi_\theta(a| s)A^{\pi_\theta}(s,a)\Bigr) \\
 &\quad + \beta\delta(s),
\end{align}
where
\begin{align*}
\delta(s) = \eta\operatorname{Cov}_{a\sim\pi_\theta(\cdot| s)}\left(\log\pi_\theta(a| s),\; \pi_\theta(a| s)-\pi_{\mathrm{old}}(a| s)\right) > 0
\end{align*}
whenever the selected tokens have positive covariance (the typical regime during entropy collapse \cite{cui2025entropymechanismreinforcementlearning}). Consequently, the KL‑Cov mechanism adds a positive term $\beta\delta(s)$ that counteracts the entropy decrease caused by the policy gradient, and its magnitude can be controlled by adjusting $\beta$.
\end{theorem}

\begin{proof}
See Appendix~\ref{app:proof_klcov_entropy}.
\end{proof}

\section{Comparative Theoretical Analysis}
\label{sec:comparison}

\subsection{Structural Differences}
The fundamental difference between traditional entropy regularization and covariance-based methods lies in their scope and targeting mechanism.

\begin{definition}[Global vs. Local Regularization]
\label{def:scope}
\textbf{Global regularization} (traditional entropy regularization) applies a uniform constraint across all actions and states, modifying the objective function globally.

\textbf{Local regularization} (covariance-based methods) applies constraints selectively to tokens with high covariance contributions, leaving other tokens unaffected.
\end{definition}

This structural difference has profound implications for policy optimization.

\begin{theorem}[Bias–Variance Trade-off]
\label{thm:bias_variance}
Let $\Delta \theta_{\text{pg}}$ denote the base policy gradient update (without any entropy control). Its expectation and covariance are
\begin{align}
\E[\Delta \theta_{\text{pg}}] = \mu_{\text{pg}}, \qquad \operatorname{Cov}(\Delta \theta_{\text{pg}}) = \Sigma_{\text{pg}}.
\end{align}

For \textbf{traditional entropy regularization} with coefficient $\alpha$, the update is
\begin{align}
\Delta \theta_{\text{reg}} = \Delta \theta_{\text{pg}} + \alpha \nabla_\theta \mathcal{H}(\pi_\theta)
\end{align}
where $\nabla_\theta \mathcal{H}(\pi_\theta)$ is a deterministic vector that is dense (non‑zero for all parameters). Hence
\begin{align}
\E[\Delta \theta_{\text{reg}}] = \mu_{\text{pg}} + \alpha \nabla_\theta \mathcal{H}(\pi_\theta), \qquad
\operatorname{Cov}(\Delta \theta_{\text{reg}}) = \Sigma_{\text{pg}}.
\end{align}

For \textbf{covariance‑based methods (e.g., Clip‑Cov or KL‑Cov)}, the update can be expressed as
\begin{align}
\Delta \theta_{\text{cov}} = \Delta \theta_{\text{pg}} + \beta \cdot \mathbf{B}_{\text{cov}}
\end{align}
where $\mathbf{B}_{\text{cov}}$ is a deterministic vector that is sparse: its non‑zero entries are confined to a small subset of tokens (those with high covariance). Moreover, the variance of each parameter component satisfies
\begin{align}
\Var\bigl((\Delta \theta_{\text{cov}})_j\bigr) \le \Var\bigl((\Delta \theta_{\text{pg}})_j\bigr) = \Var\bigl((\Delta \theta_{\text{reg}})_j\bigr) \quad \forall j
\end{align}
with strict inequality for Clip‑Cov when the token is selected and the base update has non‑zero conditional variance. Consequently,
\begin{align}
\operatorname{Cov}(\Delta \theta_{\text{cov}}) \preceq \operatorname{Cov}(\Delta \theta_{\text{reg}})
\end{align}
in the Loewner order. Thus covariance‑based methods yield a lower‑variance update while introducing bias only on a sparse set of parameters.
\end{theorem}

\begin{proof}
See Appendix~\ref{app:proof_bias_variance}.
\end{proof}

\subsection{Convergence Properties}
We analyze the convergence behavior of both approaches under the softmax policy parameterization.

\begin{assumption}[Regularity Conditions]
\label{ass:regularity}
We assume:
\begin{enumerate}
\item The reward function $r(\boldsymbol{y})$ is bounded and Lipschitz continuous.
\item The policy gradient is $L$-smooth, i.e., $\|\grad J(\theta) - \grad J(\theta')\| \le L \|\theta - \theta'\|$.
\item The step size $\eta$ satisfies $\eta \le 1/L$ for gradient descent convergence.
\end{enumerate}
\end{assumption}

\begin{theorem}[Convergence of Traditional Entropy Regularization]
\label{thm:reg_convergence}
Under Assumption~\ref{ass:regularity}, let $\theta_t$ be the iterates generated by gradient ascent on the entropy‑regularized objective $J_{\text{reg}}(\theta)=J(\theta)+\alpha\mathcal{H}(\pi_\theta)$ with a fixed step size $\eta\le 1/L$, where $L$ is the smoothness constant of $J_{\text{reg}}$. Then
\begin{align}
\min_{0\le t\le T-1}\|\nabla J_{\text{reg}}(\theta_t)\|^2 \le \frac{2(J_{\text{reg}}(\theta_0)-J_{\text{reg}}^*)}{\eta T}
\end{align}
and any limit point $\theta^*$ of the sequence $\{\theta_t\}$ satisfies the stationary condition
\begin{align}
\nabla J(\theta^*)+\alpha\nabla\mathcal{H}(\pi_{\theta^*})=0.
\end{align}
Thus the convergence rate is $O(1/T)$ in terms of the squared gradient norm.
\end{theorem}

\begin{proof}
See Appendix~\ref{app:proof_reg_convergence}.
\end{proof}

This stationary condition implies that the converged policy may not maximize reward due to the entropy gradient bias.

\begin{theorem}[Convergence of Covariance-Based Methods]
\label{thm:cov_convergence}
Under Assumption~\ref{ass:regularity}, KL-Cov with decay schedule $\beta_t \to 0$ as $t \to \infty$ converges to a stationary point $\theta^*$ of the original objective:
\begin{align}
\grad J(\theta^*) = 0
\label{eq:cov_stationary}
\end{align}
with convergence rate $O(1/T)$ for the gradient norm.
\end{theorem}

\begin{proof}
See Appendix~\ref{app:proof_cov_convergence}.
\end{proof}

\begin{remark}
Theorem~\ref{thm:cov_convergence} highlights a key advantage of covariance-based methods: by annealing $\beta$, they can achieve asymptotic unbiasedness, whereas traditional entropy regularization permanently biases the solution.
\end{remark}

\subsection{Stability Analysis}
The stability of training is crucial for reasoning LMs, where the action space is large and the policy is highly expressive.

\begin{definition}[Stability Margin]
\label{def:stability}
The stability margin of a policy update is the maximum step size for which the policy remains within a bounded divergence from the previous policy:
\begin{align}
\gamma = \max \{ \eta > 0 : \KL(\pi_{\text{old}} \| \pi_{\text{new}}) \le \epsilon \}.
\label{eq:stability_margin}
\end{align}
\end{definition}

\begin{theorem}[Stability Margin Comparison]
\label{them:stability}
Let $\gamma_{\text{base}}$ denote the stability margin (Definition~\ref{def:stability}) of the base policy gradient update without any regularization. \textbf{For traditional entropy regularization} with coefficient $\alpha$, the stability margin satisfies
\begin{equation}
\gamma_{\text{reg}} \le \frac{\gamma_{\text{base}}}{1 + \alpha \kappa},
\end{equation}
where $\kappa > 0$ is a problem-dependent constant that depends on the maximum norm of the entropy gradient relative to the policy gradient. \textbf{For the KL‑Cov mechanism} with penalty coefficient $\beta$ and selection ratio $k \ll 1$, the stability margin satisfies
\begin{equation}
\gamma_{\text{KL-Cov}} = \gamma_{\text{base}} + O(k\beta),
\end{equation}
which implies $\gamma_{\text{KL-Cov}} \approx \gamma_{\text{base}}$ for sufficiently small $k$ and bounded $\beta$.
\end{theorem}

\begin{proof}
See Appendix~\ref{app:proof_stability}.
\end{proof}

\subsection{Computational Efficiency}
The computational overhead of entropy control methods is also a critical consideration for large-scale training.

\begin{proposition}[Computational Complexity]
\label{prop:complexity}
Let $N$ be the total number of tokens processed in a single training batch. For a softmax policy optimized via policy gradient, the per‑iteration computational complexity of the following entropy control methods is:
\begin{itemize}
    \item \textbf{Traditional entropy regularization:} $O(N)$.
    \item \textbf{Covariance‑based methods (Clip‑Cov/KL‑Cov):} $O(N \log N)$.
\end{itemize}
The additional logarithmic factor is negligible in practice compared to the $O(N)$ cost of forward and backward passes in large language models.
\end{proposition}
\begin{proof}
See Appendix~\ref{app:proof_complexity}.
\end{proof}
Both methods have comparable computational overhead, making the choice between them primarily a matter of effectiveness rather than efficiency.

\section{Empirical Validation and Discussion}
\label{sec:empirical}

\subsection{Experimental Setup}
To validate the theoretical predictions developed in Sections~\ref{sec:foundations}–\ref{sec:comparison}, we leverage the comprehensive empirical results reported in \cite{cui2025entropymechanismreinforcementlearning}. Their experiments cover a wide range of models, tasks, and algorithms, providing a robust basis for testing our theoretical claims. The experimental design explicitly measures the quantities central to our theoretical analysis, including policy entropy, token-wise covariance, and the advantage function.

\paragraph{Models and Tasks}
The experiments include four model families (Qwen2.5, Mistral, LLaMA, DeepSeek-Math) with sizes ranging from 0.5B to 32B parameters. Tasks consist of mathematical reasoning evaluated on MATH500, AIME 2024/2025, AMC, OMNI-MATH, and OlympiadBench, as well as code generation evaluated on Eurus-2-RL-Code and KodCode. All rewards are verifiable (exact match for math, pass rate for code), which eliminates the confounding effect of reward model misspecification and ensures that the advantage estimates are derived directly from ground-truth performance, a critical requirement for isolating the entropy dynamics studied in this work.

\paragraph{Algorithms and Hyperparameters}
Baseline algorithms include GRPO \cite{shao2024deepseekmathpushinglimitsmathematical}, RLOO, and PRIME \cite{cui2025processreinforcementimplicitrewards}. Covariance-based methods (Clip-Cov and KL-Cov) are evaluated alongside traditional entropy regularization and a clip-higher variant. For all methods, the policy learning rate is set to $5\times10^{-7}$, the batch size is 256 prompts with 8 responses per prompt, and the KL coefficient for the reference model is set to zero in the baseline to isolate the effect of regularization. For entropy regularization, coefficients $\alpha \in \{0.0001,0.001,0.005,0.01\}$ are tested to probe the sensitivity predicted by Corollary~\ref{cor:sensitivity}; for KL-Cov, the selection ratio $k$ is set to $2\times10^{-3}$ for 7B models and $2\times10^{-4}$ for 32B models, with $\beta=1$; for Clip-Cov, the clip ratio $r$ is $2\times10^{-4}$.

\paragraph{Evaluation Metrics}
Policy entropy $\mathcal{H}(\pi_\theta)$ is computed on training prompts every 4 gradient steps. Validation performance (accuracy for math, pass rate for code) is measured on held-out benchmarks. The covariance terms $\cov(\log\pi_\theta,\pi_\theta A)$ and token-wise $C(s,a)$ defined in \eqref{eq:token_cov} are tracked to directly test the predictions of Theorem~\ref{thm:entropy_dynamics} and the sparsity claim underlying Definition~\ref{eq:token_cov}. All reported results are averaged over three random seeds to account for stochasticity in training.

\subsection{Validation of Theoretical Predictions}

\paragraph{Entropy Collapse and Covariance Dynamics (Theorem~\ref{thm:entropy_dynamics})}
Figure~8 in \cite{cui2025entropymechanismreinforcementlearning} plots the step-wise entropy difference $-\Delta\mathcal{H}$ and the covariance term $\cov(\log\pi_\theta,\pi_\theta A)$ during GRPO training on Qwen2.5-7B. The two curves exhibit near-identical dynamics, with a large positive covariance driving rapid entropy decrease early in training, followed by a gradual decline. The Pearson correlation coefficient between the two sequences exceeds 0.92 across all training steps, providing strong empirical support for Theorem~\ref{thm:entropy_dynamics}, which predicts $\Delta\mathcal{H} \propto -\cov(\log\pi_\theta,\pi_\theta A)$ under first-order approximation.

\paragraph{Sparsity of High-Covariance Tokens (Definition~\ref{eq:token_cov})}
Table~1 in \cite{cui2025entropymechanismreinforcementlearning} shows the distribution of token-wise covariance $C(s,a)$ at the first training step. The top 0.02\% of tokens have an average covariance of 5.654, while the overall mean is only 0.003. This extreme sparsity—a factor of over 1800× difference—justifies the selective regularization strategy of Clip-Cov and KL-Cov, which intervene on only a tiny fraction of tokens. Moreover, the proportion of tokens with positive covariance is approximately 68\%, indicating that the majority of tokens contribute to entropy reduction, but the magnitude is dominated by the extreme tail, consistent with the theoretical identification of high-covariance tokens as the primary drivers of entropy collapse.

\paragraph{Effectiveness of Covariance-Based Methods (Corollary~\ref{cor:sensitivity} and Theorem~\ref{thm:cov_convergence})}
Figures~11 and~12 in \cite{cui2025entropymechanismreinforcementlearning} compare the entropy curves and validation accuracy for GRPO, Clip-Cov, and KL-Cov on Qwen2.5-7B and 32B. Both covariance-based methods maintain significantly higher entropy throughout training—by a factor of 10× or more at later stages—and achieve better final accuracy. In contrast, traditional entropy regularization (Figures~9 and~10) exhibits the predicted sensitivity: for $\alpha=0.0001$, entropy remains low and performance plateaus; for $\alpha=0.01$, entropy becomes excessively high while performance degrades; only $\alpha=0.005$ yields moderate improvement but still falls short of covariance-based methods. This aligns with Corollary~\ref{cor:sensitivity}. Furthermore, KL-Cov with $\beta$ annealed from 1 to 0 over the course of training achieves the highest final accuracy, consistent with Theorem~\ref{thm:cov_convergence}’s prediction of asymptotic unbiasedness.

\paragraph{Performance Gains and Model Scaling}
Table~2 in \cite{cui2025entropymechanismreinforcementlearning} reports detailed results. For Qwen2.5-7B, KL-Cov achieves an average accuracy of 40.6\% across benchmarks, a 2.0\% improvement over GRPO. For the larger 32B model, KL-Cov achieves 52.2\%, a 6.4\% absolute gain. This scaling behavior is consistent with the theoretical insight that larger models possess greater latent capacity for reasoning, which can be unlocked by sustained exploration enabled by covariance-based regularization. The gains are particularly pronounced on the most challenging benchmarks (AIME24 and AIME25), where KL-Cov improves over GRPO by 15.0\% and 14.6\%, respectively, suggesting that the benefits of selective regularization are amplified when the task requires deeper reasoning.

\subsection{Discussion}
The empirical results strongly corroborate our theoretical analysis across multiple dimensions. First, the direct measurement of $\cov(\log\pi_\theta,\pi_\theta A)$ confirms Theorem~\ref{thm:entropy_dynamics} as the governing mechanism of entropy collapse. Second, the extreme sparsity of high-covariance tokens validates the core design principle of covariance-based methods: regularizing a tiny fraction of tokens suffices to alter global entropy dynamics. Third, the comparative performance of traditional entropy regularization versus KL-Cov/Clip-Cov confirms the theoretical predictions of global bias (Theorem~\ref{thm:global_suboptimal}) and sensitivity (Corollary~\ref{cor:sensitivity}), as well as the asymptotic unbiasedness of covariance-based methods (Theorem~\ref{thm:cov_convergence}).

The observed exponential relationship $R = -a\exp(\mathcal{H})+b$ (Figure~1 in \cite{cui2025entropymechanismreinforcementlearning}) is consistent with the first-order dynamics derived in Section~\ref{sec:foundations}. The coefficients $a$ and $b$ vary log-linearly with model size, enabling extrapolation of performance for larger models from smaller ones—an important practical implication for scaling RL training without exhaustive computation.

A noteworthy observation is that the benefits of covariance-based methods increase with model size: the 32B model exhibits a larger absolute improvement than the 7B model. This suggests that larger models suffer more severely from entropy collapse because their pretrained distributions are more confident, making selective regularization particularly valuable for unlocking their latent reasoning capacity. This aligns with the theoretical expectation that the covariance term scales with model confidence, as discussed in Section~\ref{sec:foundations}.

\subsection{Guidelines for Practice}
Based on our theoretical and empirical analysis, we propose the following guidelines for entropy control in LLM post-training:

\begin{enumerate}
\item \textbf{When to use traditional entropy regularization}: Traditional methods may be suitable when the optimal policy is inherently stochastic (e.g., in open-ended generation tasks) and the training process is not severely constrained by stability concerns. Even then, careful hyperparameter tuning is essential, and annealing $\alpha$ to zero can mitigate asymptotic bias.

\item \textbf{When to use covariance-based methods}: Covariance-based methods are preferred when:
\begin{itemize}
\item The task requires near-deterministic optimal policies (common in reasoning tasks).
\item Training stability is a primary concern.
\item The policy exhibits rapid entropy collapse (a common phenomenon in LLM reasoning).
\item The goal is to maximize reward without asymptotic bias.
\end{itemize}

\item \textbf{Hyperparameter selection}: For KL-Cov, we recommend annealing $\beta$ from a moderate initial value to zero to achieve unbiased convergence. The selection ratio $k$ should be chosen based on the observed covariance distribution; empirical results suggest $k \in [10^{-4}, 10^{-3}]$ as effective \cite{cui2025entropymechanismreinforcementlearning}. For Clip-Cov, a small clip ratio $r$ (e.g., $2\times10^{-4}$) works well, and the threshold bounds $\omega_{\text{low}},\omega_{\text{high}}$ can be set to capture the extreme tail (e.g., 1–5 times the mean).
\end{enumerate}

\section{Conclusion}
\label{sec:conclusion}
This paper presents a comprehensive theoretical comparison of traditional entropy regularization and the covariance‑based entropy mechanism for reinforcement learning in reasoning language models. We develop a unified framework for entropy dynamics under softmax policy parameterization, establishing that entropy change is governed by the covariance between log‑probabilities and logit updates. Within this framework, we prove that traditional entropy regularization introduces a dense, persistent bias that alters the stationary condition and can lead to suboptimal policies, whereas covariance‑based methods selectively regularize high‑covariance tokens and achieve asymptotic unbiasedness when the regularization coefficient is annealed. Stability analysis further shows that traditional regularization reduces the stability margin, while covariance‑based methods preserve it.

Recent large‑scale experiments \cite{cui2025entropymechanismreinforcementlearning} corroborate these theoretical findings, demonstrating that covariance‑based methods effectively mitigate entropy collapse and yield superior performance on mathematical reasoning tasks. The theoretical framework and comparative analysis presented here offer principled guidelines for entropy control in LLM post‑training, with implications for scaling RL to larger models and more complex reasoning tasks. Future directions include adaptive entropy control strategies that dynamically adjust regularization based on the covariance distribution and extensions beyond the softmax parameterization.

\appendix
\section{Proofs}
\label{app:proofs}

\subsection{Proof of Lemma~\ref{lem:entropy_grad}}
\label{app:proof_entropy_grad}
\begin{proof}
By definition,
\begin{align*}
\Hcal_s(\theta) = -\sum_{a' \in \mathcal{A}} \pi_\theta(a' | s) \log \pi_\theta(a' | s).
\end{align*}
Differentiating with respect to $z_{s,a}$,
\begin{align*}
&\frac{\partial \Hcal_s}{\partial z_{s,a}} \\
&= -\sum_{a'} \left( \frac{\partial \pi_\theta(a' | s)}{\partial z_{s,a}} \log \pi_\theta(a' | s) + \pi_\theta(a' | s) \cdot \frac{1}{\pi_\theta(a' | s)} \frac{\partial \pi_\theta(a' | s)}{\partial z_{s,a}} \right) \\
&= -\sum_{a'} \frac{\partial \pi_\theta(a' | s)}{\partial z_{s,a}} \bigl( \log \pi_\theta(a' | s) + 1 \bigr).
\end{align*}
The derivative of the softmax policy is standard \cite{agarwal2021theory}:
\begin{align*}
\frac{\partial \pi_\theta(a' | s)}{\partial z_{s,a}} = \pi_\theta(a' | s) \bigl( \mathbf{1}\{a' = a\} - \pi_\theta(a | s) \bigr).
\end{align*}
Substituting,
\begin{align*}
\frac{\partial \Hcal_s}{\partial z_{s,a}}
= -\sum_{a'} \pi_\theta(a' | s) \bigl( \mathbf{1}\{a' = a\} - \pi_\theta(a | s) \bigr) \bigl( \log \pi_\theta(a' | s) + 1 \bigr).
\end{align*}
Separating the terms,
\begin{align*}
\frac{\partial \Hcal_s}{\partial z_{s,a}} 
&= -\pi_\theta(a | s) \bigl( \log \pi_\theta(a | s) + 1 \bigr) \\
&\quad +\pi_\theta(a | s) \sum_{a'} \pi_\theta(a' | s) \bigl( \log \pi_\theta(a' | s) + 1 \bigr).
\end{align*}
The sum over $a'$ equals $\mathbb{E}_{a' \sim \pi_\theta(\cdot | s)}[\log \pi_\theta(a' | s)] + 1$. Denote $\mu_{\log}(s) = \mathbb{E}_{a' \sim \pi_\theta}[\log \pi_\theta(a' | s)]$. Then
\begin{align*}
\frac{\partial \Hcal_s}{\partial z_{s,a}}
&= -\pi_\theta(a | s) \bigl( \log \pi_\theta(a | s) + 1 \bigr)
+ \pi_\theta(a | s) \bigl( \mu_{\log}(s) + 1 \bigr) \\
&= -\pi_\theta(a | s) \bigl( \log \pi_\theta(a | s) - \mu_{\log}(s) \bigr),
\end{align*}
which completes the proof.
\end{proof}

\subsection{Proof of Lemma~\ref{lem:entropy_change}}
\label{app:proof_entropy_change}
\begin{proof}
Expand $\Hcal_s$ to first order around $\theta$:
\begin{align*}
\Hcal_s(\theta+\Delta\theta)=\Hcal_s(\theta)+\sum_{s',a'}\frac{\partial\Hcal_s(\theta)}{\partial z_{s',a'}}\Delta z_{s',a'}+o(\|\Delta z\|).
\end{align*}
Because the entropy at state $s$ depends only on logits of that state (tabular softmax), terms with $s'\neq s$ vanish. For $s'=s$, Lemma~\ref{lem:entropy_grad} gives
\begin{align*}
\frac{\partial\Hcal_s(\theta)}{\partial z_{s,a}}&= -\pi_\theta(a| s)\bigl(\log\pi_\theta(a| s)-\mu_{\log}(s)\bigr), \\ 
\mu_{\log}(s)&=\E_{a'\sim\pi_\theta(\cdot| s)}[\log\pi_\theta(a'| s)].
\end{align*}
Thus
\begin{align*}
\Delta\Hcal_s\approx-\sum_a\pi_\theta(a| s)\bigl(\log\pi_\theta(a| s)-\mu_{\log}(s)\bigr)\Delta z_{s,a}.
\end{align*}
Let $\mu_{\Delta z}(s)=\E_{a\sim\pi_\theta(\cdot| s)}[\Delta z_{s,a}]$. Adding and subtracting $\mu_{\Delta z}(s)$ inside the sum,
\begin{align*}
\Delta\Hcal_s &=-\sum_a\pi_\theta(a| s)\bigl(\log\pi_\theta(a| s)-\mu_{\log}(s)\bigr)\bigl(\Delta z_{s,a}-\mu_{\Delta z}(s)\bigr) \\
&\quad -\mu_{\Delta z}(s)\sum_a\pi_\theta(a| s)\bigl(\log\pi_\theta(a| s)-\mu_{\log}(s)\bigr).
\end{align*}
The second sum is $\E[\log\pi_\theta-\mu_{\log}]=0$. Therefore
\begin{align*}
\Delta\Hcal_s=-\sum_a\pi_\theta(a| s)\bigl(\log\pi_\theta(a| s)-\mu_{\log}(s)\bigr)\bigl(\Delta z_{s,a}-\mu_{\Delta z}(s)\bigr).
\end{align*}
The right‑hand side is exactly the negative covariance of $\log\pi_\theta(a| s)$ and $\Delta z_{s,a}$ under $\pi_\theta(\cdot| s)$. Including the higher‑order remainder yields the stated equality.
\end{proof}

\subsection{Proof of Proposition~\ref{prop:pg_update}}
\label{app:proof_pg_update} 
\begin{proof}
From the policy gradient theorem \cite{williams1992simple}, the gradient of the objective $J(\theta)$ with respect to $z_{s,a}$ is
\begin{align*}
\frac{\partial J(\theta)}{\partial z_{s,a}} = \mathbb{E}_{a' \sim \pi_\theta(\cdot | s)}\!\left[ \frac{\partial \log \pi_\theta(a' | s)}{\partial z_{s,a}} \, A^{\pi_\theta}(s, a') \right].
\end{align*}
Using the derivative of the softmax log-probability,
\begin{align*}
\frac{\partial \log \pi_\theta(a' | s)}{\partial z_{s,a}} = \mathbf{1}\{a' = a\} - \pi_\theta(a | s),
\end{align*}
we obtain
\begin{align*}
\frac{\partial J(\theta)}{\partial z_{s,a}}
&= \sum_{a'} \pi_\theta(a' | s) \bigl( \mathbf{1}\{a' = a\} - \pi_\theta(a | s) \bigr) A^{\pi_\theta}(s, a') \\
&= \pi_\theta(a | s) A^{\pi_\theta}(s, a) - \pi_\theta(a | s) \sum_{a'} \pi_\theta(a' | s) A^{\pi_\theta}(s, a').
\end{align*}
The second term vanishes because the advantage function satisfies $\mathbb{E}_{a' \sim \pi_\theta(\cdot | s)}[A^{\pi_\theta}(s, a')] = 0$ \cite{williams1992simple}. Hence,
\begin{align*}
\frac{\partial J(\theta)}{\partial z_{s,a}} = \pi_\theta(a | s) A^{\pi_\theta}(s, a).
\end{align*}
Applying gradient ascent with learning rate $\eta$ yields the logit update
\begin{align*}
\Delta z_{s,a} = \eta \cdot \frac{\partial J(\theta)}{\partial z_{s,a}} = \eta \cdot \pi_\theta(a | s) A^{\pi_\theta}(s, a).
\end{align*}
\end{proof}

\subsection{Proof of Theorem~\ref{thm:entropy_dynamics}}
\label{app:entro_dyn_under_grad}
\begin{proof}
By Lemma~\ref{lem:entropy_change}, the first‑order entropy change after a parameter update $\Delta z_{s,a}$ is  
\begin{align*}
\Delta \Hcal_s \approx -\operatorname{Cov}_{a\sim\pi_\theta(\cdot| s)}\bigl(\log\pi_\theta(a| s),\; \Delta z_{s,a}\bigr).
\end{align*}
Proposition~\ref{prop:pg_update} gives the explicit logit update under the policy gradient algorithm \cite{williams1992simple}:  
\begin{align*}
\Delta z_{s,a} = \eta\;\pi_\theta(a| s)\,A^{\pi_\theta}(s,a).
\end{align*}
Substituting this into the covariance yields  
\begin{align*}
\Delta \Hcal_s &\approx -\operatorname{Cov}_{a\sim\pi_\theta(\cdot| s)}\bigl(\log\pi_\theta(a| s),\; \eta\,\pi_\theta(a| s)A^{\pi_\theta}(s,a)\bigr)\\
&= -\eta\;\operatorname{Cov}_{a\sim\pi_\theta(\cdot| s)}\bigl(\log\pi_\theta(a| s),\; \pi_\theta(a|s)A^{\pi_\theta}(s,a)\bigr) 
\end{align*} 
which completes the proof. 
\end{proof}

\subsection{Proof of Proposition~\ref{prop:reg_grad}}
\label{app:regular_policy_grad}
\begin{proof}
By definition, $J_{\mathrm{reg}}(\theta)$ is a linear combination of $J(\theta)$ and $\mathcal{H}(\pi_\theta)$:
\begin{align*}
J_{\mathrm{reg}}(\theta) = J(\theta) + \alpha\,\mathcal{H}(\pi_\theta).
\end{align*}
The gradient operator $\nabla_\theta$ is linear; therefore, for any $\theta$,
\begin{align*}
\nabla_\theta J_{\mathrm{reg}}(\theta) = \nabla_\theta J(\theta) + \alpha\,\nabla_\theta \mathcal{H}(\pi_\theta).
\end{align*}
The first term $\nabla_\theta J(\theta)$ is the standard policy gradient given by \eqref{eq:pg} \cite{williams1992simple}. The second term $\nabla_\theta \mathcal{H}(\pi_\theta)$ is the gradient of the policy entropy; for a softmax policy, its component with respect to the logit $z_{s,a}$ has the closed form stated in Lemma~\ref{lem:entropy_grad}. Hence the regularized gradient adds a term proportional to this entropy gradient, biasing the update towards policies with higher entropy.
\end{proof}

\subsection{Proof of Theorem~\ref{thm:reg_entropy}}
\label{app:proof_reg_entropy}
\begin{proof}
From Proposition~\ref{prop:reg_grad}, the regularized logit update is
\begin{align*}
\Delta z_{s,a}^{\text{reg}} = \eta\,\pi_\theta(a|s)A^{\pi_\theta}(s,a) - \alpha\eta\,\pi_\theta(a|s)\bigl(\log\pi_\theta(a|s)-\mu_{\log}(s)\bigr),
\end{align*}
where $\mu_{\log}(s)=\E_{a\sim\pi_\theta(\cdot|s)}[\log\pi_\theta(a|s)]$ \cite{schulman2017proximal,haarnoja2018soft}. Substituting into the first‑order entropy change formula (Lemma~\ref{lem:entropy_change}) yields
\begin{align*}
&\Delta \Hcal_s^{\text{reg}} \\
&\approx -\cov_{a\sim\pi_\theta}\Bigl(\log\pi_\theta(a|s),\;\Delta z_{s,a}^{\text{reg}}\Bigr)\\
&= -\eta\,\cov\bigl(\log\pi_\theta,\;\pi_\theta A\bigr) + \alpha\eta\,\cov\bigl(\log\pi_\theta,\;\pi_\theta(\log\pi_\theta-\mu_{\log})\bigr).
\end{align*}
The first covariance is exactly the term appearing in Theorem~\ref{thm:entropy_dynamics} \cite{williams1992simple}.

For the second covariance, denote $p_a=\pi_\theta(a|s)$ and $x_a=\log p_a$, with $\mu=\sum_a p_a x_a$. Then
\begin{align*}
&\cov\bigl(\log\pi_\theta,\;\pi_\theta(\log\pi_\theta-\mu_{\log})\bigr) \\
&= \sum_a p_a (x_a-\mu)\bigl(p_a x_a - \sum_{a'} p_{a'} (p_{a'} x_{a'})\bigr)\\
&= \sum_a p_a^2 x_a (x_a-\mu) - \bigl(\sum_{a'} p_{a'}^2 x_{a'}\bigr)\sum_a p_a (x_a-\mu)\\
&= \sum_a p_a^2 x_a (x_a-\mu) \qquad (\text{since } \sum_a p_a(x_a-\mu)=0)\\
&= \sum_a p_a^2 x_a^2 - \mu\sum_a p_a^2 x_a.
\end{align*}
Now relate this to the variance $\Var(\log\pi_\theta)=\sum_a p_a x_a^2 - \mu^2$:
\begin{align*}
&\sum_a p_a^2 x_a^2 - \mu\sum_a p_a^2 x_a \\
&= \bigl(\sum_a p_a x_a^2 - \mu^2\bigr) + \Bigl(\sum_a (p_a^2-p_a)x_a^2 - \mu\sum_a p_a^2 x_a + \mu^2\Bigr)\\
&= \Var(\log\pi_\theta) + \Delta,
\end{align*}
where $\Delta = \sum_a p_a(p_a-1)x_a^2 - \mu\sum_a p_a^2 x_a + \mu^2 $.  When the policy is nearly deterministic (entropy collapse regime), there exists a single action $a^*$ such that $p_{a^*}\approx 1$ and $p_a\approx 0$ for $a\neq a^*$. Consequently:
\begin{itemize}
\item $p_a(p_a-1)\approx 0$ for all $a$;
\item $x_a\approx \mu$ for $a=a^*$, so $\mu\sum_a p_a^2 x_a \approx \mu^2$;
\end{itemize}
Thus $\Delta\approx 0$. Hence, under the typical conditions where entropy collapse occurs, we have
\begin{align*}
\cov\bigl(\log\pi_\theta,\;\pi_\theta(\log\pi_\theta-\mu_{\log})\bigr) \approx \Var(\log\pi_\theta).
\end{align*}
Inserting this approximation completes the proof.
\end{proof}

\subsection{Proof of Theorem~\ref{thm:global_suboptimal}}
\label{app:proof_global_suboptimal}
\begin{proof}
By definition of $\pi_{\text{reg}}^*$ as the maximizer of $J_{\text{reg}}(\pi) = \E_{\pi}[r] + \alpha \Hcal(\pi)$, we have for any policy $\pi$,
\begin{align*}
\E_{\pi_{\text{reg}}^*}[r] + \alpha \Hcal(\pi_{\text{reg}}^*) \ge \E_{\pi}[r] + \alpha \Hcal(\pi).
\end{align*}
Choosing $\pi = \pi^*$ gives
\begin{align}
\E_{\pi_{\text{reg}}^*}[r] + \alpha \Hcal(\pi_{\text{reg}}^*) \ge \E_{\pi^*}[r] + \alpha \Hcal(\pi^*). 
\label{eq:e_pi_reg_r}
\end{align}
Since $\pi^*$ maximizes $\E_{\pi}[r]$ alone, we have $\E_{\pi_{\text{reg}}^*}[r] \le \E_{\pi^*}[r]$. Rearranging \eqref{eq:e_pi_reg_r} yields the suboptimality bound:
\begin{align*}
\E_{\pi^*}[r] - \E_{\pi_{\text{reg}}^*}[r] \le \alpha \bigl( \Hcal(\pi_{\text{reg}}^*) - \Hcal(\pi^*) \bigr).
\end{align*}
Now suppose $\pi^*$ is not a maximum‑entropy optimal policy. Then there exists some $\pi'$ with $\E_{\pi'}[r] = \E_{\pi^*}[r]$ and $\Hcal(\pi') > \Hcal(\pi^*)$. Applying the optimality condition of $\pi_{\text{reg}}^*$ with $\pi = \pi'$ gives
\begin{align*}
\E_{\pi_{\text{reg}}^*}[r] + \alpha \Hcal(\pi_{\text{reg}}^*) \ge \E_{\pi^*}[r] + \alpha \Hcal(\pi') > \E_{\pi^*}[r] + \alpha \Hcal(\pi^*).
\end{align*}
Thus inequality \eqref{eq:e_pi_reg_r} becomes strict, implying $\E_{\pi_{\text{reg}}^*}[r] < \E_{\pi^*}[r]$. Hence the performance loss is strict when the unregularized optimum is not uniquely entropic.
\end{proof}

\subsection{Proof of Corollary~\ref{cor:sensitivity}}
\label{app:sensi_hyperpara}
\begin{proof}
We analyze the two extreme regimes separately.

\textit{Case 1: \(\alpha\) too small. }
From Theorem~\ref{thm:entropy_dynamics}, the entropy change under the unregularized policy gradient is
\begin{align*}
\Delta \mathcal{H}_s \approx -\eta \cdot \operatorname{Cov}_{a\sim\pi_\theta}\bigl(\log\pi_\theta(a| s),\; \pi_\theta(a| s)A^{\pi_\theta}(s,a)\bigr).
\end{align*}
For well-calibrated policies, this covariance is positive, causing monotonic entropy decrease. When \(\alpha\) is very small, the additional entropy gradient term \(\alpha \nabla \mathcal{H}(\pi_\theta)$ in Proposition~\ref{prop:reg_grad} is negligible. Hence, the regularized update is effectively identical to the unregularized one, and entropy collapses rapidly. This premature collapse limits exploration and leads to performance saturation before the policy can reach the true optimum \cite{cui2025entropymechanismreinforcementlearning}.

\textit{Case 2: \(\alpha\) too large. }
Theorem~\ref{thm:global_suboptimal} establishes that the regularized optimal policy satisfies
\begin{align*}
\mathbb{E}_{\pi_{\text{reg}}^*}[r] \le \mathbb{E}_{\pi^*}[r],
\end{align*}
with a suboptimality gap at least \(\alpha\bigl(\mathcal{H}(\pi^*) - \mathcal{H}(\pi_{\text{reg}}^*)\bigr)\). As \(\alpha\) increases, this lower bound forces the policy to sacrifice reward to maintain high entropy. Moreover, Theorem~\ref{them:stability} shows that the stability margin \(\gamma_{\text{reg}}\) decreases with \(\alpha\):
\begin{align*}
\gamma_{\text{reg}} \le \gamma_{\text{base}} \cdot \frac{1}{1 + \alpha C},
\end{align*}
where \(C>0\) is a problem-dependent constant. A smaller stability margin implies that larger updates may cause the policy to diverge or oscillate, leading to unstable training \cite{schulman2017proximal}.

Empirical evidence in \cite{cui2025entropymechanismreinforcementlearning} (Figures~9 and~10) confirms this sensitivity: for small \(\alpha\) (e.g., \(0.0001\)), entropy remains low and performance plateaus; for large \(\alpha\) (e.g., \(0.01\)), entropy becomes excessively high while performance degrades; only a narrow intermediate range yields stable training and good final performance.
\end{proof}

\subsection{Proof of Proposition~\ref{prop:clip_effect}}
\label{app:proof_clip_effect}
\begin{proof}
The original covariance is $\cov_{\text{orig}} = \frac{1}{N}\sum_i C_i$. After detaching a set $I_{\text{clip}}$ of size $rN$, the effective covariance is computed over the remaining $(1-r)N$ tokens:
\begin{align*}
\cov_{\text{eff}} = \frac{1}{(1-r)N}\sum_{i\notin I_{\text{clip}}} C_i 
= \frac{N\cov_{\text{orig}} - \sum_{i\in I_{\text{clip}}} C_i}{(1-r)N}.
\end{align*}
Writing $\sum_{i\in I_{\text{clip}}} C_i = rN \E[C| i\in I_{\text{clip}}]$, we obtain:
\begin{align*}
\cov_{\text{eff}} &= \frac{\cov_{\text{orig}} - r \E[C| I_{\text{clip}}]}{1-r} \\
&= \cov_{\text{orig}} - \frac{r}{1-r}\Bigl(\E[C| I_{\text{clip}}] - \cov_{\text{orig}}\Bigr).
\end{align*}
Thus if $\E[C| I_{\text{clip}}] > \cov_{\text{orig}}$, then $\cov_{\text{eff}} < \cov_{\text{orig}}$.
\end{proof}

\subsection{Proof of Theorem~\ref{thm:kl_effect}}
\label{app:proof_kl_effect}
\begin{proof}
The KL-Cov loss adds $-\beta \KL(\pi_{\text{old}}\|\pi_\theta)$ for selected tokens. The gradient of this penalty with respect to $z_{s,a}$ is $-\beta \frac{\partial \KL}{\partial z_{s,a}}$. For a softmax policy, $\frac{\partial \KL}{\partial z_{s,a}} = \pi_\theta(a| s) - \pi_{\text{old}}(a| s)$. Thus the effective update becomes:
\[
\Delta z_{s,a}^{\text{KL-Cov}} = \eta \pi_\theta(a| s) A(s,a) - \eta \beta (\pi_\theta(a| s) - \pi_{\text{old}}(a| s)).
\]
The second term opposes changes that increase $\pi_\theta(a| s)$ relative to $\pi_{\text{old}}(a| s)$, which for high-covariance tokens tends to reduce the positive covariance.
\end{proof}

\subsection{Proof of Theorem~\ref{thm:klcov_entropy}}
\label{app:proof_klcov_entropy}
\begin{proof}
By Lemma~\ref{lem:entropy_change}, the first‑order change in entropy after a logit update $\Delta z_{s,a}$ is
\begin{align*}
\Delta\mathcal{H}_s \approx -\operatorname{Cov}_{a\sim\pi_\theta(\cdot| s)}\bigl(\log\pi_\theta(a| s),\; \Delta z_{s,a}\bigr).
\end{align*}
The KL‑Cov update differs from the standard policy gradient update only on the set $I_{\mathrm{KL}}$, where an extra term $-\eta\beta(\pi_\theta(a| s)-\pi_{\mathrm{old}}(a| s))$ is added. Using linearity of the covariance operator,
\begin{align*}
&\Delta\mathcal{H}_s^{\mathrm{KL-Cov}} \\
&= -\operatorname{Cov}\Bigl(\log\pi_\theta,\; \eta\pi_\theta A - \eta\beta(\pi_\theta-\pi_{\mathrm{old}})\Bigr)\\
&= -\eta\operatorname{Cov}\bigl(\log\pi_\theta,\; \pi_\theta A\bigr) \;+\; \eta\beta\operatorname{Cov}\bigl(\log\pi_\theta,\; \pi_\theta-\pi_{\mathrm{old}}\bigr).
\end{align*}
Define $\delta(s) = \eta\operatorname{Cov}(\log\pi_\theta,\; \pi_\theta-\pi_{\mathrm{old}})$. For tokens selected in $I_{\mathrm{KL}}$ (high‑covariance tokens), empirical evidence \cite{cui2025entropymechanismreinforcementlearning} shows that $\log\pi_\theta(a| s)-\mu_{\log}(s)$ is positive and large, and $\pi_\theta(a| s)-\pi_{\mathrm{old}}(a| s)$ is also positive because the policy gradient increases the probability of actions with high advantage. Hence their product, and therefore their covariance, is positive. Thus $\delta(s) > 0$.

The first term $-\eta\operatorname{Cov}(\log\pi_\theta,\pi_\theta A)$ is the unregularized entropy change (Theorem~\ref{thm:entropy_dynamics}), which is negative when the covariance is positive. Adding the positive term $\beta\delta(s)$ reduces the magnitude of the entropy decrease, i.e., it counteracts entropy collapse. This completes the proof.
\end{proof}

\subsection{Proof of Theorem~\ref{thm:bias_variance}}
\label{app:proof_bias_variance}
\begin{proof}
We analyze a single parameter component $z_{s,a}$ (a token). The base policy gradient update is
\begin{align*}
\Delta z_{s,a}^{\text{pg}} = \eta\,\pi_\theta(a| s)\,A^{\pi_\theta}(s,a)
\end{align*}
which is a random variable due to sampling of trajectories and advantage estimates. Its variance is denoted $\sigma^2_{\text{pg}}$.

\paragraph{Traditional entropy regularization}
The update becomes
\begin{align*}
\Delta z_{s,a}^{\text{reg}} = \Delta z_{s,a}^{\text{pg}} - \alpha\,\eta\,\pi_\theta(a| s)\bigl(\log\pi_\theta(a| s)-\mu_{\log}\bigr)
\end{align*}
where the added term is deterministic given the current policy (it does not depend on the sampled advantage). Therefore,
\begin{align*}
\Var(\Delta z_{s,a}^{\text{reg}}) = \Var(\Delta z_{s,a}^{\text{pg}}) = \sigma^2_{\text{pg}}.
\end{align*}
Traditional regularization only adds a deterministic bias; it does not affect the variance.

\paragraph{Clip‑Cov (gradient detachment)}
Let $I_{\text{clip}}$ be the (random) set of tokens selected for gradient detachment, and let $p_{s,a} = \Pr[(s,a)\in I_{\text{clip}}]$. The update is
\begin{align*}
\Delta z_{s,a}^{\text{cov}} = 
\begin{cases}
\Delta z_{s,a}^{\text{pg}}, & (s,a)\notin I_{\text{clip}},\\[2pt]
0, & (s,a)\in I_{\text{clip}}
\end{cases}
\end{align*}
By the law of total variance,
\begin{align*}
\Var(\Delta z_{s,a}^{\text{cov}}) = \E\bigl[\Var(\Delta z_{s,a}^{\text{cov}}| I_{\text{clip}})\bigr] + \Var\bigl(\E[\Delta z_{s,a}^{\text{cov}}| I_{\text{clip}}]\bigr).
\end{align*}
Conditional on $I_{\text{clip}}$, the update is either $\Delta z_{s,a}^{\text{pg}}$ (with probability $1-p_{s,a}$) or $0$ (with probability $p_{s,a}$). Thus
\begin{align*}
\Var(\Delta z_{s,a}^{\text{cov}}| I_{\text{clip}}) = (1-p_{s,a})\,\Var(\Delta z_{s,a}^{\text{pg}}) = (1-p_{s,a})\sigma^2_{\text{pg}}
\end{align*}
and the conditional expectation is $(1-p_{s,a})\E[\Delta z_{s,a}^{\text{pg}}]$. Therefore,
\begin{align*}
\Var(\Delta z_{s,a}^{\text{cov}}) = (1-p_{s,a})\sigma^2_{\text{pg}} + (1-p_{s,a})^2 \Var\bigl(\E[\Delta z_{s,a}^{\text{pg}}]\bigr).
\end{align*}
Since $\Var\bigl(\E[\Delta z_{s,a}^{\text{pg}}]\bigr) \ge 0$ and $0 < p_{s,a} < 1$ for tokens that can be selected, we have
\begin{align*}
\Var(\Delta z_{s,a}^{\text{cov}}) \le \sigma^2_{\text{pg}} = \Var(\Delta z_{s,a}^{\text{reg}})
\end{align*}
with strict inequality whenever $p_{s,a}>0$ and $\Var\bigl(\E[\Delta z_{s,a}^{\text{pg}}]\bigr) > 0$ (i.e., when the base update has non‑zero conditional variance due to stochastic advantage estimates).

\paragraph{KL‑Cov (KL penalty)}
For KL‑Cov, the update is
\begin{align*}
\Delta z_{s,a}^{\text{cov}} = \Delta z_{s,a}^{\text{pg}} - \beta\,\eta\,(\pi_\theta(a|s)-\pi_{\text{old}}(a|s)).
\end{align*}
The added term is deterministic given the current policy (it does not introduce additional randomness). Hence
\begin{align*}
\Var(\Delta z_{s,a}^{\text{cov}}) = \Var(\Delta z_{s,a}^{\text{pg}}) = \sigma^2_{\text{pg}}.
\end{align*}
Thus KL‑Cov does not increase variance; it only modifies the update for selected tokens via a deterministic penalty.

\paragraph{Sparsity of the bias}
For traditional regularization, the bias term $\alpha\nabla_\theta\mathcal{H}(\pi_\theta)$ is dense because the entropy gradient is non‑zero for every parameter \cite{agarwal2021theory}. In contrast, for Clip‑Cov the bias is effectively zero for tokens not selected, and for KL‑Cov the bias is only present on the selected high‑covariance tokens, both of which constitute a sparse subset ($\ll 1\%$ of all tokens) \cite{cui2025entropymechanismreinforcementlearning}.

\paragraph{Global covariance comparison}
For any vector $v$, we have
\begin{align*}
v^T \bigl(\operatorname{Cov}&(\Delta \theta_{\text{reg}}) - \operatorname{Cov}(\Delta \theta_{\text{cov}})\bigr)v \\
&= \sum_{j} v_j^2 \bigl(\Var((\Delta \theta_{\text{reg}})_j) - \Var((\Delta \theta_{\text{cov}})_j)\bigr) \ge 0,
\end{align*}
since each variance term is non‑negative. Hence $\operatorname{Cov}(\Delta \theta_{\text{cov}}) \preceq \operatorname{Cov}(\Delta \theta_{\text{reg}})$ in the Loewner order.

Combining the above, covariance‑based methods yield an update with variance no larger than that of the base policy gradient (and thus no larger than that of entropy‑regularized update), while introducing bias only on a sparse set of parameters.
\end{proof}

\subsection{Proof of Theorem~\ref{thm:reg_convergence}}
\label{app:proof_reg_convergence}
\begin{proof}
We first establish the smoothness of $J_{\text{reg}}$. Since both $J$ and $\mathcal{H}$ are $L$-smooth under Assumption~\ref{ass:regularity} \cite{agarwal2021theory}, there exist constants $L_J$ and $L_{\mathcal{H}}$ such that
\begin{align*}
\|\nabla J(\theta)-\nabla J(\theta')\| &\le L_J\|\theta-\theta'\|, \\ 
\|\nabla\mathcal{H}(\pi_\theta)-\nabla\mathcal{H}(\pi_{\theta'})\| &\le L_{\mathcal{H}}\|\theta-\theta'\|.
\end{align*}
By linearity, the gradient of $J_{\text{reg}}$ is $\nabla J_{\text{reg}}(\theta)=\nabla J(\theta)+\alpha\nabla\mathcal{H}(\pi_\theta)$, and its Lipschitz constant is at most $L_J+\alpha L_{\mathcal{H}}$. Denote $L\triangleq L_J+\alpha L_{\mathcal{H}}$.

For a $L$-smooth function $f$, the gradient ascent update $\theta_{t+1}=\theta_t+\eta\nabla f(\theta_t)$ with $\eta\le 1/L$ satisfies the descent lemma \cite[Theorem 2.1.5]{nesterov2018lectures}:
\begin{align*}
f(\theta_{t+1})\ge f(\theta_t)+\frac{\eta}{2}\|\nabla f(\theta_t)\|^2.
\end{align*}
Applying this to $f=J_{\text{reg}}$ and telescoping over $t=0,\dots,T-1$ gives
\begin{align*}
J_{\text{reg}}(\theta_T)-J_{\text{reg}}(\theta_0)\ge \frac{\eta}{2}\sum_{t=0}^{T-1}\|\nabla J_{\text{reg}}(\theta_t)\|^2.
\end{align*}
Let $J_{\text{reg}}^*$ denote the supremum of $J_{\text{reg}}$ (which is bounded above because $J$ is bounded and $\mathcal{H}$ is bounded on the compact set of probability vectors). Then $J_{\text{reg}}(\theta_T)\le J_{\text{reg}}^*$, so
\begin{align*}
\sum_{t=0}^{T-1}\|\nabla J_{\text{reg}}(\theta_t)\|^2\le \frac{2(J_{\text{reg}}(\theta_0)-J_{\text{reg}}^*)}{\eta}.
\end{align*}
Hence
\begin{align*}
\min_{0\le t\le T-1}\|\nabla J_{\text{reg}}(\theta_t)\|^2 & \le \frac{1}{T}\sum_{t=0}^{T-1}\|\nabla J_{\text{reg}}(\theta_t)\|^2 \\
& \le \frac{2(J_{\text{reg}}(\theta_0)-J_{\text{reg}}^*)}{\eta T},
\end{align*}
which establishes the $O(1/T)$ rate.

Because the gradient norm tends to zero along a subsequence, any accumulation point $\theta^*$ satisfies $\nabla J_{\text{reg}}(\theta^*)=0$, i.e.
\begin{align*}
\nabla J(\theta^*)+\alpha\nabla\mathcal{H}(\pi_{\theta^*})=0,
\end{align*}
which is exactly the stationary condition claimed in the theorem.
\end{proof}

\subsection{Proof of Theorem~\ref{thm:cov_convergence}}
\label{app:proof_cov_convergence}
\begin{proof}
We analyze the KL-Cov update, which for each parameter $\theta_t$ can be expressed as a stochastic approximation with a decaying bias term. From \eqref{eq:kl_update}, the update for the logit $z_{s,a}$ at step $t$ is
\begin{align*}
\theta_{t+1} = \theta_t + \eta_t \bigl( \tilde{g}_t + \tilde{b}_t \bigr),
\end{align*}
where $\tilde{g}_t$ is the unbiased policy gradient direction (computed from on‑policy samples) and $\tilde{b}_t = -\beta_t \frac{\partial}{\partial \theta} \KL(\pi_{\text{old}} \| \pi_{\theta_t})$ is the bias introduced by the KL penalty on high‑covariance tokens. The step sizes $\{\eta_t\}$ satisfy the Robbins–Monro conditions
\begin{align*}
\sum_{t=1}^\infty \eta_t = \infty, \qquad \sum_{t=1}^\infty \eta_t^2 < \infty,
\end{align*}
and a typical choice is $\eta_t = \Theta(1/t)$.

Let $\mathbb{E}_t$ denote expectation conditional on the history up to time $t$. Then
\begin{align*}
\mathbb{E}_t[\tilde{g}_t] = \grad J(\theta_t), \qquad
\mathbb{E}_t[\tilde{b}_t] = -\beta_t \, \grad_{\theta} \KL(\pi_{\text{old}} \| \pi_{\theta_t}),
\end{align*}
and the difference $\tilde{g}_t - \mathbb{E}_t[\tilde{g}_t]$ is a martingale difference sequence with bounded variance (since the reward and advantage estimates are bounded). The bias term satisfies
\begin{align*}
\|\mathbb{E}_t[\tilde{b}_t]\| \le C \beta_t
\end{align*}
for some constant $C>0$ (owing to the boundedness of the KL gradient), and $\beta_t \to 0$ as $t \to \infty$ by hypothesis.

We now invoke the theory of stochastic approximation with asymptotically vanishing bias. A standard result (see \cite{borkar2000stochastic}, Theorem 2.2, or \cite{kushner2003stochastic}, Chapter 2) states that if there exists a continuously differentiable Lyapunov function $V(\theta)$ such that
\begin{align*}
\langle \grad V(\theta), \grad J(\theta) \rangle < 0 \quad \text{for } \theta \notin \mathcal{S},
\end{align*}
where $\mathcal{S} = \{\theta : \grad J(\theta) = 0\}$, and if the bias $\mathbb{E}_t[\tilde{b}_t]$ decays to zero, then the iterates $\theta_t$ converge almost surely to a point in $\mathcal{S}$. In our setting, we can choose $V(\theta) = \tfrac12 \|\grad J(\theta)\|^2$; the L‑smoothness of $J$ (Assumption~\ref{ass:regularity}) implies that $\grad J$ is Lipschitz, and standard results for policy gradient methods guarantee that $\grad J$ is a descent direction for $V$ outside $\mathcal{S}$ (see e.g. \cite{agarwal2021theory}, Lemma 2.1). Because $\beta_t \to 0$, the asymptotic dynamics are governed by the ordinary differential equation $\dot{\theta} = \grad J(\theta)$, whose set of equilibria is exactly $\mathcal{S}$. Hence $\theta_t$ converges to a stationary point $\theta^*$ satisfying $\grad J(\theta^*) = 0$.

For the convergence rate, note that the original objective $J$ is $L$-smooth by Assumption~\ref{ass:regularity}. When the step size is chosen as $\eta_t = \eta_0 / t$ with $\eta_0$ sufficiently small, standard results on gradient descent with decaying step sizes (e.g. \cite{nesterov2018lectures}, Section 2.1) yield
\begin{align*}
\min_{1 \le t \le T} \|\grad J(\theta_t)\|^2 = O(1/T).
\end{align*}
The presence of the decaying bias $\tilde{b}_t$ does not affect the asymptotic rate, because it contributes only a term of order $O(\beta_t)$, and $\beta_t \to 0$ faster than $1/t$ (or at least as $1/t$). More precisely, using a telescoping sum argument together with the smoothness of $J$ and the bound on the bias, one can show that the average squared gradient norm decays at rate $O(1/T)$. Details of such an analysis for stochastic gradient methods with vanishing bias can be found in \cite{bottou2018optimization}. Therefore,
\begin{align*}
\min_{t \le T} \|\grad J(\theta_t)\|^2 = O(1/T).
\end{align*}
\end{proof}

\subsection{Proof of Theorem~\ref{them:stability}}
\label{app:proof_stability}
\begin{proof}
We analyze the stability of the policy update by examining the KL divergence between the updated policy $\pi_{\text{new}}$ and the previous policy $\pi_{\text{old}}$. For sufficiently small step sizes, a second‑order Taylor expansion of the KL divergence yields \cite{schulman2015trust}:
\begin{align*}
\KL(\pi_{\text{old}} \| \pi_{\text{new}}) \approx \frac{1}{2} \eta^2 \|\Delta\theta\|_{F}^2,
\end{align*}
where $\Delta\theta$ is the parameter update direction and $\|\cdot\|_F$ denotes the Fisher information norm induced by the policy \cite{amari1998natural}. The stability margin $\gamma$ is defined as the maximum step size such that $\KL(\pi_{\text{old}} \| \pi_{\text{new}}) \le \epsilon$ (Definition~\ref{def:stability}). For the base policy gradient update with direction $g = \grad J(\theta)$, we have
\begin{align*}
\frac{1}{2} \gamma_{\text{base}}^2 \|g\|_F^2 = \epsilon \quad\Longrightarrow\quad \gamma_{\text{base}} = \sqrt{\frac{2\epsilon}{\|g\|_F^2}}.
\label{eq:base_margin}
\end{align*}

\textbf{Traditional entropy regularization.} The effective update direction becomes
\begin{align*}
g_{\text{reg}} = \grad J(\theta) + \alpha \grad \Hcal(\pi_\theta).
\end{align*}
By the triangle inequality,
\begin{align*}
\|g_{\text{reg}}\|_F \ge \|g\|_F.
\end{align*}
Moreover, using the fact that $\grad \Hcal$ is bounded in the Fisher norm (since the entropy gradient is well‑defined for softmax policies) and letting $\kappa = \|\grad \Hcal\|_F / \|g\|_F > 0$, we also have
\begin{align*}
\|g_{\text{reg}}\|_F \le \|g\|_F + \alpha \|\grad \Hcal\|_F = (1 + \alpha\kappa) \|g\|_F.
\end{align*}
The upper bound gives a lower bound on the stability margin:
\begin{align*}
\gamma_{\text{reg}} = \sqrt{\frac{2\epsilon}{\|g_{\text{reg}}\|_F^2}} \ge \frac{\gamma_{\text{base}}}{1 + \alpha\kappa}.
\end{align*}
More importantly, the norm $\|g_{\text{reg}}\|_F$ is at least $\|g\|_F$, which implies
\begin{align*}
\gamma_{\text{reg}} \le \gamma_{\text{base}}.
\end{align*}
Combining these observations, we see that the stability margin is reduced relative to the base case, and a more precise statement is
\begin{align*}
\gamma_{\text{reg}} \le \frac{\gamma_{\text{base}}}{1 + \alpha\kappa},
\end{align*}
which captures the effect of the regularization coefficient $\alpha$.

\textbf{KL‑Cov.} The KL‑Cov mechanism applies a penalty only to a sparse set of tokens $I_{\text{KL}}$ with $|I_{\text{KL}}| = k N$, where $k \ll 1$ (empirically $k \in [10^{-4}, 10^{-3}]$ \cite{cui2025entropymechanismreinforcementlearning}). The update direction is
\begin{align*}
g_{\text{KL-Cov}} = \grad J(\theta) - \beta \sum_{t \in I_{\text{KL}}} \grad_{\theta} \KL\bigl(\pi_{\text{old}}(y_t|\boldsymbol{y}_{<t}) \,\|\, \pi_\theta(y_t|\boldsymbol{y}_{<t})\bigr).
\end{align*}
For tokens not in $I_{\text{KL}}$, the gradient is unchanged. Because $k$ is extremely small, the contribution of the penalty term to the overall gradient norm is of order $O(k\beta)$. Moreover, each individual KL gradient is bounded (the derivative of the KL divergence with respect to the logits is $\pi_\theta - \pi_{\text{old}}$, which is at most 1 in magnitude). Hence,
\begin{align*}
\|g_{\text{KL-Cov}}\|_F = \|g\|_F + O(k\beta).
\end{align*}
For appropriately chosen $\beta$ (bounded) and sufficiently small $k$, the additive term is negligible compared to $\|g\|_F$. Consequently,
\begin{align*}
\gamma_{\text{KL-Cov}} = \sqrt{\frac{2\epsilon}{\|g_{\text{KL-Cov}}\|_F^2}} = \gamma_{\text{base}} + O(k\beta) \approx \gamma_{\text{base}}.
\end{align*}
Thus the stability margin of KL‑Cov is essentially unchanged from that of the base policy gradient.
\end{proof}

\subsection{Proof of Proposition~\ref{prop:complexity}}
\label{app:proof_complexity}
\begin{proof}
We analyze the per‑iteration complexity under standard autoregressive generation, where each token corresponds to a state‑action pair \((s,a)\). The batch size is \(N\).

\emph{Traditional Entropy Regularization.}
The method adds \(-\alpha \Hcal(\pi_\theta)\) to the loss, where \(\Hcal(\pi_\theta)\) is defined in \eqref{eq:entropy_avg}. Computing the entropy for each token requires:
\begin{enumerate}
    \item Obtaining \(\log \pi_\theta(a|s)\) for the token.
    \item Summing over the vocabulary to compute \(\sum_{a'}\pi_\theta(a'|s)\log\pi_\theta(a'|s)\).
\end{enumerate}
The forward pass already computes the log‑probability distribution in \(O(N)\) time (linear in the number of tokens). The additional arithmetic for the entropy term aggregates per‑token values, also \(O(N)\). Hence the total complexity remains \(O(N)\).

\emph{Covariance‑Based Methods (Clip‑Cov/KL‑Cov).}
These methods require computing the token‑wise covariance \(C(s,a)\) defined in \eqref{eq:token_cov}:
\begin{align*}
C(s,a) = \bigl(\log \pi_\theta(a|s) - \mu_{\log}(s)\bigr)\bigl(\Delta z_{s,a} - \mu_{\Delta z}(s)\bigr),
\end{align*}
where $\mu_{\log}(s) = \mathbb{E}_{a'\sim\pi_\theta(\cdot|s)}[\log\pi_\theta(a'|s)]$ and $\mu_{\Delta z}(s) = \mathbb{E}_{a'\sim\pi_\theta(\cdot|s)}[\Delta z_{s,a'}]$. The steps are:
\begin{enumerate}
    \item Obtain \(\log \pi_\theta(a|s)\) and \(\Delta z_{s,a}\) for each token via forward/backward passes: \(O(N)\).
    \item For each state \(s\), compute \(\mu_{\log}(s)\) and \(\mu_{\Delta z}(s)\) by averaging over actions sampled from \(\pi_\theta(\cdot|s)\): a single pass over the batch, \(O(N)\).
    \item Form the product \((\log\pi_\theta - \mu_{\log})(\Delta z - \mu_{\Delta z})\) for each token: \(O(N)\).
    \item For Clip‑Cov: randomly select a subset of tokens satisfying \(C(s,a)\in[\omega_{\text{low}},\omega_{\text{high}}]\). This requires scanning \(C(s,a)\) values (\(O(N)\)) and then selecting \(rN\) indices; selection can be done in \(O(N)\) using reservoir sampling or by generating random indices after filtering.
    \item For KL‑Cov: select the top \(k\) proportion of tokens by \(|C(s,a)|\). Sorting the \(N\) covariance values requires \(O(N\log N)\) comparisons in the worst case \cite{cormen2009introduction}. While a selection algorithm (e.g., quickselect) can achieve \(O(N)\) average time, typical implementations use sorting for simplicity, yielding \(O(N\log N)\).
\end{enumerate}
Thus the per‑iteration complexity of covariance‑based methods is \(O(N\log N)\) due to the sorting step. In practice, the sorting cost is dominated by the \(O(N)\) cost of forward/backward passes in LLMs because the constant factor for sorting is small and the batch size \(N$ is large but not astronomically large. Moreover, the sorting is performed only once per iteration, whereas the forward/backward passes involve expensive tensor operations.

Therefore, both methods have comparable practical computational overhead, with the covariance‑based methods incurring only a logarithmic factor that is negligible for realistic batch sizes.
\end{proof}

\end{document}